\documentclass{article} 
\usepackage{times}
\usepackage{url}
\usepackage{times}
\usepackage{epsfig}
\usepackage{graphicx}
\usepackage{amsmath, amsthm, amssymb}
\usepackage{algorithm}
\usepackage{algorithmic}
\usepackage{natbib}
\usepackage{bibentry}

\newtheorem{lemma}{Lemma}
\newtheorem{theorem}{Theorem}
\newtheorem{cor}{Corollary}

\newcommand{\ignore}[1]{}
\newcommand{\nobibentry}[1]{{\let\nocite\ignore\bibentry{#1}}}

\begin{document}
\nobibliography*

\title{A Simple Expression for Mill's Ratio of the Student's $t$-Distribution}
\author{Francesco Orabona\\
Yahoo! Labs\\New York, NY, USA\\
{\tt\small francesco@orabona.com}
}

\maketitle

\begin{abstract}
I show a simple expression of the Mill's ratio of the Student's $t$-Distribution. I use it to prove Conjecture 1 in \nobibentry{AuerCF02}.
\end{abstract}

We first need the following technical lemma.
\begin{lemma}
\label{lemma:exp_ratio}
For all $x> 0$ we have
\[
x\frac{e^\frac{x^2}{2}}{\sqrt{e^{x^2}-1}}\leq x+1
\]
\end{lemma}
\begin{proof}
Denote by $y=e^{x^2}$, with $y>1$.
We have
\begin{align}
\frac{\sqrt{y}}{\sqrt{y-1}}-1=\frac{\sqrt{y}-\sqrt{y-1}}{\sqrt{y-1}} = \frac{1}{\sqrt{y-1}(\sqrt{y}-\sqrt{y-1})} \leq \frac{1}{\sqrt{y-1}}~.
\end{align}
Hence we have
\begin{align}
x\left(\frac{e^\frac{x^2}{2}}{\sqrt{e^{x^2}-1}} -1\right) \leq x \frac{1}{{\sqrt{e^{x^2}-1}}} \leq 1,
\end{align}
where in the last inequality we used $\exp(z)-1 \geq z$.
\end{proof}

The following theorem provides simple upper bounds to the Mill's ratio of a $t$-Student.
\begin{theorem}
\label{theo:t1}
Let $f_\nu(x)$ the pdf of a Student's $t$-distribution with $\nu$ degrees of freedom. Then,
for any $\nu\geq0$, we have
\[
\frac{\int_{a}^{+\infty} f_\nu(x) dx}{f_\nu(a)} \leq \sqrt{1+\frac{a^2}{\nu}} \left( \frac{1}{2} + \frac{1}{\sqrt{\nu}} \right), \textit{ if } a\geq0
\]
\[
\frac{\int_{a}^{+\infty} f_\nu(x) dx}{f_\nu(a)} \leq \sqrt{1+\frac{a^2}{\nu}} \left( 1 + \frac{1}{\sqrt{\nu}} \right), \textit{ if } a<0
\]
\end{theorem}
\begin{proof}
The first stated inequality holds for $a=0$, for the symmetry of the $t$-Student distribution, hence we can safely assume $a\neq0$. We have that
\begin{align*}
P[X \geq a] = C_\nu \int_{a}^{+\infty} \left(1+\frac{x^2}{\nu}\right)^{-\frac{\nu+1}{2}} dx,
\end{align*}
where $C_\nu=\frac{\Gamma(\frac{\nu+1}{2})}{\sqrt{\nu \pi } \Gamma(\frac{\nu}{2})}$.
With the change of variable $z=\sqrt{\nu \log (1+\frac{x^2}{\nu})}$, we have
\begin{align*}
P[X \geq a] &= C_\nu \int_{\sqrt{\nu \log (1+\frac{a^2}{\nu})}}^{+\infty} \ 
e^{-\frac{z^2}{2}} \frac{z e^\frac{z^2}{2 \nu}}{\sqrt{\nu (e^\frac{z^2}{\nu}-1)}} dz\\
&\leq C_\nu \int_{\sqrt{\nu \log (1+\frac{a^2}{\nu})}}^{+\infty} \ 
e^{-\frac{z^2}{2}} \left(\frac{z}{\sqrt{\nu}}+1\right) dz\\
&= C_\nu \left( \frac{1}{\sqrt{\nu}} \left(1+\frac{a^2}{\nu}\right)^{-\frac{\nu}{2}} + \int_{\sqrt{\nu \log (1+\frac{a^2}{\nu})}}^{+\infty} \ 
e^{-\frac{z^2}{2}} dz\right),
\end{align*}
where in the inequality we used Lemma~\ref{lemma:exp_ratio}.
We now use the facts that $\int_{a}^{+\infty} e^{-\frac{x^2}{2}} dx \leq \frac{1}{2} e^{-\frac{a^2}{2}}$ for $a\geq0$ and 
$\int_{a}^{+\infty} e^{-\frac{x^2}{2}} dx \leq e^{-\frac{a^2}{2}}$ for $a\leq 0$ to have the stated bounds.
\end{proof}

The following Corollary is a slightly better version of Conjecture 1 in \citet{AuerCF02}. I could not find a proof of it in any paper so I decided to give a simple proof for it.
\begin{cor}
Let $X$ be a Student's $t$ random variable with $\nu$ degrees of freedom. Then, for
$0 \leq a \leq \sqrt{2 (\nu +1.22) }$ and $\nu\geq0$, we have
\[
P[X \geq a] \leq e^\frac{-a^2}{4}~.
\]
\end{cor}
\begin{proof}
First observe that $C_\nu \left( \frac{1}{2} + \frac{1}{\sqrt{\nu}} \right) \leq K = 0.543$, as it can be verified numerically.
Using the first result of Theorem~\ref{theo:t1} we have
\begin{align*}
P[X \geq a] 
\leq K e^{-\frac{\nu}{2}\log\left(1+\frac{a^2}{\nu}\right)} 
=  e^{\log(K) - \frac{\nu}{2}\log\left(1+\frac{a^2}{\nu}\right)}
\leq  e^{\log(K) - \frac{a^2\nu}{a^2+2\nu}},
\end{align*}
where in the last inequality we used the fact that $\log(x+1)\geq \frac{2x}{x+2}$.
Hence the statement of the theorem is equivalent to find the upper bound on $a^2$ such that
\[
\log(K) - \frac{a^2 \nu}{a^2+2\nu} \leq -\frac{a^2}{4},
\]
that in turn is equivalent to
\[
a^4 +2a^2 (2 \log(K) - \nu) + 8\nu \log(K)  \leq 0
\]
Solving the quadratic equation we have
\begin{align}
a^2 
&\leq \nu - 2 \log(K)  + \sqrt{\nu^2 + 4 \log^2(K) - 12 \nu \log(K)}
\end{align}
Hence, the condition $a^2 \leq 2 \nu - 4 \log(K)$ satisifies the inequality above. Using the value of $K$ we have the stated bound.
\end{proof}

\bibliographystyle{plainnat}
\bibliography{learning}

\end{document}